\documentclass[sigconf]{acmart}
\AtBeginDocument{%
  }

\setcopyright{acmlicensed}
\copyrightyear{2026}
\acmYear{2026}
\acmDOI{XXXXXXX.XXXXXXX}
\acmConference[CIKM '26]{The 35th ACM International Conference on Information and Knowledge Management}{November 9--11, 2026}{Rome, Italy}
\acmISBN{978-1-4503-XXXX-X/2018/06}

\usepackage[most]{tcolorbox}
\usepackage{multirow}
\usepackage[table]{xcolor}
\definecolor{verylightgray}{gray}{0.95}

\tcbset{
  bluebox/.style={
    colback=blue!3,
    colframe=blue!60!black,
    coltitle=white,
    fonttitle=\bfseries,
    boxrule=0.6pt,
    arc=2mm,
    left=1mm,
    right=1mm,
    top=1mm,
    bottom=1mm
  }
}

\newtheorem{consequence}{Consequence}

\newcommand{\vecT}{\boldsymbol{\theta}}
\newcommand{\matW}{\mathbf{W}}
\newcommand{\vecH}{\mathbf{h}}

\newcommand{\vecP}{\mathbf{p}}
\newcommand{\vecX}{\mathbf{x}}
\newcommand{\vecY}{\mathbf{y}}
\newcommand{\vecW}{\mathbf{w}}
\newcommand{\batchSet}{\mathcal{B}}
\newcommand{\pCoef}{C_{\hat{p}}}
\newcommand{\yCoef}{C_{y}}

\begin{document}

\title{Batch Size or Negatives? A Selection Rule for Memory-Constrained Recommender Training}

\author{Artyom Sabitov}
\authornote{All authors contributed equally to this research.}
\affiliation{%
  \institution{Moscow Independent Research Institute of Artificial Intelligence}
  \city{Moscow}
  \country{Russia}}
\affiliation{
  \institution{Intellectual data analysis and predictive modeling institute}
  \city{Moscow}
  \country{Russia}}
\email{artyom.a.sabitov@gmail.com}

\author{Daniil Volkov}
\authornotemark[1]
\affiliation{%
  \institution{Applied AI Institute}
  \city{Moscow}
  \country{Russia}
}
\email{danvol121@gmail.com}

\author{Alexey Zaytsev}
\authornotemark[1]
\affiliation{%
  \institution{Applied AI Institute,}
  \city{Moscow}
  \country{Russia}
}
\affiliation{
  \institution{Risk department}
  \city{Moscow}
  \country{Russia}}
\email{likzet@gmail.com}

\renewcommand{\shortauthors}{Sabitov et al.}

\begin{abstract}
Large-scale neural recommender systems are typically trained with a softmax cross-entropy objective over the full item vocabulary. 
For a typical large number of possible items $K$, the final classification layer dominates memory, requiring $O(nK)$ logits and gradients to materialize for a batch of $n$ examples. 
Sampled softmax reduces this cost by restricting the objective to only $k \ll K$ candidate negative items, resulting in an $O(nk)$ memory. 
However, for a fixed budget $B = n k$, it remains unclear whether one should prioritize larger batches or the inclusion of more negative items.

We address this question by analyzing sampled-softmax training under a fixed memory constraint. 
Under standard smoothness and variance assumptions, our theoretical evidence suggests that the fastest convergence arises from an $ n \sim B, k \sim 1$ allocation.
So, an actionable rule is to include as many objects as possible given computational constraints.

Our theory is supported by controlled synthetic and synthetic and four real sequential recommendation benchmarks, including MovieLens-20M. 
The suggested configuration achieve faster convergence and better final recommendation quality than imbalanced alternatives within the same memory constraint. 
These findings provide a theoretical and empirical foundation for configuring memory during the training of recommender systems.
\footnote{Code, reproducibility materials, and all scripts for generating figures are available at \url{https://anonymous.4open.science/r/LimitedMemoryRule-BBFB}}
\end{abstract}



\maketitle

\section{Introduction}

Recommender systems are a central application area of modern machine learning, powering search, ranking, advertising, media platforms, and e-commerce services~\cite{ricci2021recommender,raza2026comprehensive}. 
In recent years, sequential recommendation has been strongly influenced by transformer-based architectures such as SASRec~\cite{kang2018self} and BERT4Rec~\cite{sun2019bert4rec}, which model user histories as sequences and predict the next relevant item from a large catalog. 
A common and highly effective training formulation treats recommendation as a multi-class classification problem over all items: for each user context, the model predicts a probability distribution over a catalog of size $K$ using a softmax cross-entropy objective during training~\cite{kang2018self,sun2019bert4rec,klenitskiy2023turning,wu2024effectiveness}.

This formulation is statistically attractive but computationally expensive. 
When the item catalog is large, the final classification layer requires materializing logits and gradients of size proportional to the batch size times the number of classes. 
For a batch of $n$ examples, the full softmax objective therefore incurs an $O(nK)$ memory and compute footprint. 
In large-catalog recommender systems, this final softmax layer often becomes the dominant memory bottleneck, limiting the feasible batch size and slowing down training. 
This issue is particularly pronounced for transformer-based sequential recommenders, where strong architectures are increasingly trained on large datasets and large item vocabularies~\cite{klenitskiy2023turning,mezentsev2024scalable,zhelnin2026faster}.

A standard way to reduce this cost is to replace the full softmax with a sampled objective. 
Instead of computing the normalization over all $K$ items, sampled softmax evaluates the loss only on a subset of $k \ll K$ candidate classes~\cite{jean2015using,blanc2018adaptive,rawat2019sampled}. 
A particularly efficient variant uses in-batch negatives, where positive examples from other mini-batch items are reused as negative candidates. 
Sampling-based objectives are also widely used with bias-correction techniques, such as the logQ correction, to compensate for non-uniform item sampling probabilities~\cite{yi2019sampling,khrylchenko2025logq}. 
Class sampling methods reduce the memory cost from $O(nK)$ to approximately $O(nk)$, making large-catalog training feasible.

However, this introduces a practical question that is not well understood: under a fixed memory budget, should we allocate memory to a larger batch size or to more sampled classes? 
More precisely, if the available memory allows approximately
\[
    B = n \cdot k
\]
logits to be processed per optimization step, then increasing $n$ reduces mini-batch noise by averaging over more training examples, while increasing $k$ improves the approximation of the full softmax by using more negative classes. 
These two choices reduce different sources of stochasticity, and prioritizing only one of them may leave the other as the dominant source of optimization noise. 
Despite the practical importance of this trade-off, current training recipes usually choose $n$ and $k$ by heuristic tuning~\cite{petrov2023gsasrec,mezentsev2024scalable,zhelnin2026faster}.

In this paper, we study this trade-off from the perspective of stochastic first-order optimization. 
Modern recommender systems are typically trained using SGD or adaptive variants such as Adam~\cite{kingma2015adam}. 
For such methods, convergence behavior is strongly affected by the variance of the stochastic gradient estimator~\cite{bottou2018optimization}. 
For example, for a strongly convex and smooth objective and a stochastic gradient variance at most $M$, standard results~\cite{shamir2013stochastic} imply that for SGD at iteration $t$, the regret $R_t$, a difference between the loss at the current iteration $F_t$ and the minimum $F^*$, is such that $R_t \sim O(\frac{M}{T})$.
Thus, the asymptotic error scales linearly with the gradient variance $M$. 
Analogous dependencies on stochastic-gradient noise parameters also appear in modern convergence analyses of Adam-type methods~\cite{wang2024convergence}. 
This motivates a direct analysis of how the gradient variance depends on both the batch size $n$ and the number of sampled classes $k$. 
Unlike the standard mini-batch setting, sampled-softmax training involves two coupled sources of randomness: sampling of training examples and sampling of classes. 
To the best of our knowledge, this two-source variance trade-off under a fixed memory constraint has not been explicitly analyzed in the context of recommender training.

We provide a theoretical convergence analysis focusing on the final softmax layer, which can be viewed as a multi-class logistic regression model on top of learned sequence representations. 
This abstraction permits a tractable theoretical treatment, while still capturing the main memory bottleneck in large-catalog recommendation. 
Our analysis shows that the gradient variance decomposes into \emph{a mini-batch component} controlled by the number of objects $n$ and \emph{a class-sampling component} controlled by the number of sampled negative classes $k$. 
Balancing these two components leads to maximization of the number of considered objects:
\[
    n \asymp B, k \asymp 1.
\]
However, in practice computations become efficient for $k \geq n $ and we use the rule $ n = k = \sqrt{B}$
Our theoretical finding appears to be consistent with empirical results in sequential recommendation: both batch composition and the number of negatives substantially affect optimization dynamics and final ranking quality~\cite{sun2019bert4rec,klenitskiy2023turning,petrov2023gsasrec,wu2024effectiveness}. 
Further empirical real-data experiments provides a more targeted investigation of the $n$-$k$ balance effect.

To sum up, our contributions are as follows:
\begin{itemize}  
    \item \textbf{Framework for convergence analysis for stochastic first order methods.} We formulate sampled-softmax recommender training as stochastic optimization under a fixed memory constraint $nk \leq B$, where $n$ is the batch size and $k$ is the number of sampled classes. This leads to a theoretical framework for analyzing stochastic gradients with two sources of randomness: mini-batch sampling and class sampling. 
    For the last-layer logistic-regression model, this framework yields an explicit variance-based criterion for choosing $n$ and $k$, applicable to both vanilla SGD and Adam~optimizers.
    
    \item \textbf{Theoretically optimal  allocation rule.} We derive a practical allocation, suggesting that we should prioritize sampling examples to sampling classes in typical applied settings. 
    This gives a guideline for configuring sampled-softmax training without an expensive grid search. Given the manifestation rule for gradients, we use $n = k = \sqrt{B}$ in practice.
    
    \item \textbf{Empirical validation of the our rule.} We examine the theoretical prediction experimentally on sequential recommendation benchmarks, including MovieLens-1M and Gowalla. 
    Across SGD and Adam optimizers, proposed configurations converge faster and often achieve better final recommendation quality than strongly imbalanced alternatives under the same memory budget. 
    While we admit that larger-scale experiments would further improve the work, here we focus on a single strong architecture and common datasets to lay an empirical foundation for our theoretical findings.
\end{itemize}

\begin{tcolorbox}[bluebox]
Overall, our theoretical results suggest that efficient large-catalog recommender training depends not only on the choice of architecture, sampling objective, or optimizer, but also on how memory is allocated across the competing sources of stochasticity introduced by sampled-softmax optimization.
\end{tcolorbox}

\section{Related Work}

\subsubsection*{Large-Catalog Training in Sequential Recommendation}

Modern sequential recommender systems are commonly formulated as classification problems over the entire item catalog. 
In transformer-based models such as SASRec and BERT4Rec, this requires computing the softmax over hundreds of thousands or even millions of items. 
Under such conditions, the full cross-entropy objective becomes the primary computational and memory bottleneck, since it requires materializing logits of size $n \times L \times K$, where $n$ is the batch size, $L$ is the sequence length, and $K$ is the catalog size. 
Specifically, this bottleneck was quantified by Mezentsev et al.~\cite{mezentsev2024scalable}, who profiled SASRec training with full cross-entropy for different catalog sizes. 
Their measurements show that CE-related activations grow from $2.86$ GB for a $10$K-item catalog to $28.61$ GB for $100$K items and $286.10$ GB for $1$M items. 
In contrast, all other measured components, including weights, gradients, optimizer states, and non-CE activations, remain much smaller, totaling only $0.25$ GB, $0.34$ GB, and $1.19$ GB for the same catalog sizes. 
Thus, for a $1$M-item catalog, CE-related activations are about $240\times$ larger than all other measured training components combined, making the final softmax computation the dominant memory bottleneck in large-catalog sequential recommendation.

Early approaches typically addressed this issue using negative sampling and binary objectives. 
Recent work goes further, demonstrating that the choice of loss function and negative sampling strategy has a critical impact on the quality of sequential recommendations. 
In particular, Klenitskiy and Vasilev~\cite{klenitskiy2023turning} showed that much of the reported superiority of BERT4Rec over SASRec is explained not by the architecture itself, but by the use of full softmax cross-entropy instead of BCE loss with a single negative sample. 
Their experiments demonstrate that, under the same objective, SASRec can outperform BERT4Rec both in recommendation quality and training efficiency.

This line of research was further extended in gSASRec~\cite{petrov2023gsasrec}, where Petrov and Macdonald analyze the overconfidence effect arising from negative sampling during training. 
The authors show that the standard BCE objective systematically overestimates the probabilities of positive interactions, especially when using a small number of negative samples. To mitigate this effect, they propose a generalized BCE objective and substantially increase the number of negatives used during training.

\subsubsection*{Sampled Softmax and Bias Correction}

Thus, an alternative to BCE-based objectives is sampled softmax, which approximates the full softmax using only a subset of negative classes. 
A particularly popular approach is the use of in-batch negatives, where items from the current mini-batch are reused as negative examples. This approach scales efficiently and makes effective use of GPU parallelism, but introduces bias toward popular items.

To correct this bias, Yi et al.~\cite{yi2019sampling} proposed the logQ correction based on importance weighting with respect to the probability of an item appearing in the batch. This approach has become a standard component of large-scale retrieval and recommendation systems.

Subsequent work demonstrated that the quality of sampled softmax strongly depends on the sampling distribution and the method used to correct the resulting bias. In particular, Blanc and Rendle~\cite{blanc2018adaptive} investigated adaptive sampling schemes for approximating the softmax distribution, while Rawat et al.~\cite{rawat2019sampled} proposed Random Fourier Softmax to reduce the bias of sampled-softmax estimators in large output spaces. Wu et al.~\cite{wu2024sampledsoftmax} further showed that sampled softmax possesses several desirable properties for recommendation systems, including more effective utilization of hard negatives and reduced popularity bias compared to pointwise objectives.

More recently, Khrylchenko et al.~\cite{khrylchenko2025logq} revisited the classical logQ correction and demonstrated that the standard deviation remains biased due to incorrect treatment of the positive item contribution in the sampled-softmax denominator. The authors proposed a refined correction formula and reported improvements on retrieval and recommendation benchmarks.

Overall, the recent literature suggests that increasing the number of negative samples generally leads to more stable optimization and improved ranking quality. However, this also directly increases memory consumption and the computational cost of each training step.

\subsubsection*{Memory-Efficient Cross-Entropy Objectives}

Alongside research on sampling strategies, there has been increasing interest in memory-efficient approximations of the cross-entropy objective. The main goal of these methods is to avoid computing the full logit tensor while preserving the benefits of full-softmax training.

Mezentsev et al.~\cite{mezentsev2024scalable} proposed Scalable Cross-Entropy (SCE), where logit computation is replaced with approximate maximum inner product search (MIPS), significantly reducing memory usage without noticeable degradation in recommendation quality. In RECE~\cite{gusak2024rece}, the authors employ a locality-sensitive hashing-like approximation to reduce the computational cost of the final classification layer. Similar ideas are also used in other scalable recommendation pipelines based on approximate retrieval and efficient softmax approximation.

Another line of work modifies the parameterization of the item space itself. For example, Zivic et al.~\cite{zivic2024scaling} propose eliminating the trainable item embedding table and instead constructing item representations using a fixed feature encoder. This decouples the number of model parameters from the catalog size and facilitates training with sampled softmax objectives.

Taken together, these works demonstrate that recommendation models with extremely large item catalogs can be trained substantially more efficiently without computing the full softmax over all items, while maintaining comparable or even superior recommendation quality.

\subsubsection*{Scaling Laws and Resource Allocation}

In recent years, research on large-scale recommendation has increasingly shifted from isolated architectural improvements toward questions of optimal resource allocation. Following the work of Kaplan et al.~\cite{kaplan2020scaling}, it became common to analyze neural network training through the lens of scaling laws that relate model quality to model size, dataset size, and computational budget.

Similar scaling behaviors were later observed in sequential recommendation models. In particular, Zivic et al.~\cite{zivic2024scaling} showed that transformer-based recommenders exhibit stable power-law relationships with respect to both the number of model parameters and the number of training interactions. These findings suggest the existence of optimal trade-offs between different components of the computational budget.

However, most prior work considers either scaling model size, scaling dataset size, or selecting an appropriate sampling strategy. The question of how to optimally allocate memory between batch size and the number of negative samples under a fixed memory budget remains largely unexplored.

\subsubsection*{Our Contribution Relative to Prior Work}

The work most closely related to ours is that of Zhelnin et al.~\cite{zhelnin2026faster}, who empirically observe a trade-off between batch size and the number of negatives during training of sequential recommenders. Nevertheless, the existing literature still lacks a theoretical understanding of how memory should be optimally distributed between these two quantities.

In this work, we formulate sampled-softmax training as a stochastic optimization problem under the constraint $B = nk$, where $n$ denotes the batch size and $k$ is the number of sampled negatives. We analyze the contributions of mini-batch noise and class-sampling noise to the variance of the gradient estimator and study how the balance between these two sources affects convergence behavior under fixed memory constraints.

\section{Methods}

\subsection{Preliminaries}

\subsubsection{Recommender systems training}

We consider the standard sequential recommendation setting. 
Let $\mathcal{I} = \{1, \ldots, K\}$ be an item catalog of size $K$. 
For each user, we observe a sequence of past interactions
\[
    (i_1, i_2, \ldots, i_L),
    \qquad i_\ell \in \mathcal{I},
\]
and the goal is to predict the next relevant item $i_{L+1}$. 
Equivalently, for every user context, the model outputs a probability distribution over all catalog items and is trained to assign high probability to the observed target item.

Let $f_{\boldsymbol{\theta}}$ be a sequential recommender model with parameters $\boldsymbol{\theta}$, for example, a transformer-based architecture such as SASRec. 
Given a user history, the model produces a representation vector
\[
    \mathbf{h} = f_{\vecT}(i_1,\ldots,i_L) \in \mathbb{R}^d
\]
for a training example. 
The final prediction layer maps this representation to logits over the item catalog:
\[
    \mathbf{l} = \matW \mathbf{h},
    \qquad
    \matW \in \mathbb{R}^{K \times d},
\]
where each row of $\matW$ corresponds to one item. 
The predicted distribution is then obtained using the softmax function:
\[
    \mathbf{p}
    =
    \mathrm{softmax}(\mathbf{l}).
\]
If $\mathbf{y} \in \{0, 1\}^K$ is the one-hot vector of the target item, the usual full-softmax cross-entropy loss over a mini-batch $\batchSet$ is
\[
    \mathcal{L}(\vecT,\matW)
    =
    -\frac{1}{n}
    \sum_{j \in \batchSet}
    \mathbf{y}_{j}^{\top} \log \mathbf{p}_{j}.
\]
In practice, the parameters $\vecW = (\boldsymbol{\theta},\mathbf{W})$ are optimized by stochastic first-order methods. 

At each iteration, these methods use a stochastic gradient computed on a mini-batch rather than the full dataset. 
Additionally, for large item catalogs, computing the full softmax requires materializing logits for all $K$ classes for each example in the batch, leading to memory and computational costs proportional to $nK$. 
Large-scale recommender training often replaces the full item set with a sampled subset of $k \ll K$ classes. 
This gives a sampled-softmax or negative-sampling objective whose memory footprint is proportional to $nk \ll nK$.
Thus, in this scenario, stochasticity comes from two sources of uncertainty: selection of a subset of classes and selection of a batch.


The following subsections recall standard convergence results for SGD and Adam. 
These results show that the convergence behavior of stochastic first-order methods is governed by the variance of the gradient estimator. 
This motivates our later analysis of how the gradient variance depends on the pair $(n, k)$.

\subsubsection{SGD}
Consider stochastic gradient descent (SGD) for minimization of an objective function $F (\vecW): \mathbb{R}^d \to \mathbb{R}$:
\[
    \vecW_{t + 1} = \vecW_t - \alpha_t \mathbf{g}_t,
\]
where $\vecW_t \in \mathbb{R}^d$ is the parameter vector at iteration $t$, $\alpha_t > 0$ is the step size, and $\mathbf{g}_t$ is a stochastic gradient estimate for $F$. 
Denote $\vecW^* \in \arg\min_\matW F(\vecW)$, $F^* := F(\vecW^*)$ and $\vecW_1$ the starting point of the SGD.
The goal of SGD is to make regret $R_t = F(\vecW_t) - F^*$ small and converge for $t$ increasing to infinity.


\begin{theorem}[from \cite{bottou2018optimization}]
Let $F$ be $c$-strongly convex and $L$-smooth. Suppose SGD uses a conditionally unbiased gradient estimate $\mathbb{E}[\mathbf{g}_t \mid \vecW_t] = \nabla F(\vecW_t)$ with uniformly bounded conditional variance $\mathbb{E}\big[\|\mathbf{g}_t - \nabla F(\vecW_t)\|_2^2 \mid \vecW_t\big] \leq M$. With a constant step size $\alpha$ satisfying $0 < \alpha \leq \frac{1}{L}$, we have:
\begin{displaymath}
    \mathbb{E} R_t \leq \left(1 - \alpha c\right)^{t - 1} R_1 + \frac{\alpha L M}{2c}.
\end{displaymath}
Consequently, the asymptotic expected suboptimality is bounded by
\begin{displaymath}
    \limsup_{t \to \infty} \mathbb{E} R_t \leq \frac{\alpha L M}{2c}.
\end{displaymath}

For diminishing step sizes $\alpha_t = \frac{\beta}{\gamma + t}$ with $\beta > \frac{1}{c}$, the variance directly affects the convergence rate constant:
\begin{displaymath}
    \mathbb{E} R_t \leq \frac{\nu}{\gamma + t}, \,\,\,
    \nu = \max\left\{\frac{\beta^2 L M}{2 (\beta c - 1)}, (\gamma + 1) R_1 \right\}.
\end{displaymath}
\end{theorem}

Thus, the target regret at iteration $t$ of SGD has the form $O\left(\frac{M}{t} \right)$ and linearly depends on the upper bound for the noise variance $M$.
Reducing $M$ improves the convergence ratio for stochastic gradient descent.
Informally, we say that.


\subsubsection{SGD with momentum in non-smooth case}

For stochastic momentum methods, nonconvex convergence analyses typically bound stationarity rather than regret. 
Under smoothness and bounded-variance assumptions, Yan et al.~\cite{yan2018unified} prove that stochastic momentum methods satisfy
\[
    \min_{0\leq k\leq t}
    \mathbb{E}\|\nabla F(\vecW_k)\|^2
    =
    O\left(
        \frac{R_1 / C + C (G^2 + M)}
        {\sqrt{t}}
    \right),
\]
where $G$ bounds the gradient norm.
Again, for fixed optimization constants, the noise-dependent part of the stationarity bound is constrained from above by $M /\sqrt{t}$.
There we also see a similar dependence of $M$, and an important component for the efficiency of the optimization remains the minimization of the stochastic gradient variance.

As SGD with momentum is a special case of Adam~\cite{kingma2015adam}, those upper bounds also characterize the behaviour of it.
For additional results, see~\cite{wang2024convergence} and~\cite{carmon2020lower} with the overall observations of the linear dependence on the stochastic gradient variance remaining true.

\subsubsection{Variance Reduction}

The above results imply that accelerating stochastic optimization requires unbiased gradient estimates with minimal variance $M$. 
So, the main question for us is how the choice of hyperparameters affects the stochastic gradient's variance and how to minimize it.

\subsection{Formal Problem Setup}

While a neural network typically consists of many layers, we would focus on the last layer for several reasons and reduce the problem to the estimation of parameters for multiclass logistic regression.
Firstly, due to the backpropagation structure, the reduction in the variance of the gradient across the weights of the last layer ensures that the variance of the entire gradient is reduced.
Secondly, most of the parameters in a typical large transformer model lie in the last layer, for example, see LLMs~\cite{wijmans2025cut} or modern sequential recommendation models~\cite{zhelnin2026faster}.
Finally, adding more layers would clutter the analytical derivations and harm the clarity of the results.

We consider a sample of pairs $\mathcal{D} = \{(\vecX_j, \vecY_j)\}_{j = 1}^N$.  
Let $\vecX \in \mathbb{R}^d$ be the input feature vector to the last layer, and $\mathbf{y} \in \{0, 1\}^K$ a one‑hot vector of the true class among $K$ possible ones.

Assume that $\mathbf{y}_{\beta} \sim \mathrm{Multinomial}(\boldsymbol{\pi}_{\beta})$ with $\boldsymbol{\pi}_{\beta} = \mathrm{softmax}(\mathbf{W}^* \mathbf{x}_{\beta})$, where $\mathbf{W}^*$ is the unknown weight matrix.

The last layer is a multi‑class logistic regression with parameters $\{W_{ij}\}$ followed by a softmax. Its output is a vector $\mathbf{p}$ on the $K$-dimensional probability simplex $\Delta^K$. For a training batch, we use the cross‑entropy loss:
\begin{equation}
\mathcal{L} = - \frac{1}{n} \sum_{j \in \batchSet} \mathbf{y}_j^{T} \log \mathbf{p}_j,
\end{equation}
where $n$ is the batch size, $\log \mathbf{p}$ is the element‑wise logarithm of the predicted probabilities, and $\batchSet$ is the set of indexes for the current batch.

For $\mathbf{l} = \mathbf{W}\mathbf{x}$ being the logits and $\mathbf{h}$ their exponentials: $h_i = e^{l_i}$,
\begin{equation}
\label{eq:softmax_via_h}
\mathbf{p} = \mathrm{softmax}(\mathbf{l}) = \frac{1}{\sum_{i=1}^{K} h_i} \mathbf{h}.
\end{equation}

The total number of classes is $K$. At each training step, we randomly select a subset $\mathbb{S}$ of classes with $|\mathbb{S}| = k < K$. The true class of every example in the batch is forced to belong to $\mathbb{S}$. Computing the softmax only over $\mathbb{S}$ yields a biased vector $\mathbf{p}' \in \Delta^K$ that has zeros for the omitted classes.

Given a memory budget $B$, we must choose $n$ and $k$ such that
\begin{equation}
B = n \cdot k = \mathrm{const}.
\end{equation}

Note that in practical scenarios $k \ge \sqrt{B}$. 
$\mathbb{S}$ must contain the true class of each of the $n$ examples, we have $k \ge n$ (assuming all true classes are distinct). 
Together with $B = nk$, this implies $k \ge \sqrt{B}$.

For convenience, denote by $\mathrm{pos}_j$ the target class of example $j$, and by $\mathbb{S}_j^{\mathrm{neg}} = \mathbb{S} \setminus \{\mathrm{pos}_j\}$ the set of selected negative classes for that~example.

\subsection{Unbiased Gradient Estimate}

To apply standard SGD convergence results, we require the gradient estimator
to be unbiased. 
The following theorem provides a correction that ensures
unbiasedness under class sampling.

Let the expectation of the exponentiated logit $h_{\mathrm{neg}, j} = \exp \left(l_{\mathrm{neg}, j} \right)$ for the negative class~be
\begin{equation}
    h_{\mathrm{neg}, j} = \mathbb{E}\left[h_i \mid i \in \mathbb{S}_{j}^{\mathrm{neg}}\right],
\end{equation}
and set $\alpha_{j} = \frac{h_{\mathrm{pos}, j}}{h_{\mathrm{neg},j}}$, where $h_{\mathrm{pos},j}$ is the exponent of the logit of the target class.


\begin{theorem}
\label{th:unbias}
Denote $\mathbf{p_j''} = \mathrm{diag}(\mathbf{c_j}) \mathbf{p_j'}$, where $\mathbf{c_j}$ for the true class $c_{\mathrm{pos}, j} = \frac{\alpha_{j} + k}{\alpha_{j} + K}$, and for a negative class $
    c_{\mathrm{neg}, j} = \frac{K}{k} \cdot \frac{\alpha_{j} + k}{\alpha_{j} + K}$.    

Then, $\mathbb{E}\mathbf{p_{j}''} = \mathbb{E}\mathbf{p_{j}}$ and using $\mathbf{p_{j}''}$ instead of $ \mathbf{p_{j}'}$ give us unbiased gradient estimation.
\end{theorem}

The proof is provided in Appendix.
In practice, $\alpha_{j}$ is unknown, but we estimate it using the current logits
with $\hat{h}_{\mathrm{neg},j} = \frac{1}{|\mathbb{S}_{j}^{\mathrm{neg}}|} \sum_{i \in \mathbb{S}_{j}^{\mathrm{neg}}} h_{j i}$,
$\hat{\alpha}_{j} = \frac{h_{\mathrm{pos}, j}}{\hat{h}_{\mathrm{neg}, j}}$.

\subsection{Minimizing Variance of the Gradient Estimate}

There are two goals in this subsection: to obtain the variance of the corrected gradient and minimize it to obtain the maximum convergence speed.

When analyzing the gradient variance, we rely on approximate methods such as the Felton–Wilkinson approximation and the delta method. Their applicability is justified under the following technical assumptions:
\begin{itemize}
    \item[A1.] The logits $l_i$ are normally distributed with $\mathbb{E} l_i = 0$.
    \item[A2.] The variance of $l_i$, $\sigma_l^2 < C$ with $C = 1.2$.
    \item[A3.] $\frac{h_{\mathrm{pos}}}{K} \ll h_{\mathrm{neg}}$.
\end{itemize}

An extended overview of A1 for modern deep neural networks is available in~\cite{wolinski2025gaussian}, with Gaussianity proved in certain important cases.
For A2, we can control the variance by appropriately normalizing the previous layers, e.g., with LayerNorm~\cite{ba2016layer}. 
A3 holds due to inherent label noise in recommendation systems; this assumption also becomes non-restrictive for large values of $K$.

Again, our goal is to provide variance for the corrected sampled-softmax gradient of the last-layer weight matrix:
\[
    \mathbf{g}_{n,k}
    =
    -\frac{1}{n}
    \sum_{j\in\mathcal{B}}
    (\mathbf{p}''_j-\mathbf{y}_j)\mathbf{x}_j^\top.
\]
Here, we consider independent sampling in the mini-batch.
The following statement holds.

\begin{tcolorbox}[bluebox]

\begin{theorem}[stochastic gradient variance, informal]
\label{th:full_gradient_variance}
Assume that A1-A3 hold. 
Then the full last-layer stochastic-gradient variance satisfies
\[
    \operatorname{Var}(\mathbf{g}_{n, k})
    \leq
    \frac{\yCoef}{n}
    +
    \frac{\pCoef}{k} + \mathrm{const}
\]
for some positive constants $\yCoef \sim O(\frac{1}{B K})$ and $\pCoef \sim O(\frac{1}{K})$. 
\end{theorem}

\end{tcolorbox}

Exact values for $\yCoef$ and $\pCoef$ as well as the proof are given in the Appendix.

Given this result, the following theorem answers our main question on the value of the variance given the pair $(n, k)$.

\begin{tcolorbox}[bluebox]

\begin{theorem}
\label{th:optimal_n_k}
    Assume A1-A3 hold. 
    Then the minimum gradient variance for SGD is attained for the pair $(n*, k*)$ such that 
    \[
        n* \asymp B, \qquad k* \asymp 1.
    \]
    Under a constraint $k \geq \sqrt{B}$:
    \[
        n* \asymp  k* \asymp \sqrt{B}.
    \]
\end{theorem}
\end{tcolorbox}

\begin{consequence}
    To ensure maximum convergence speed one should minimize variance and, thus, select maximum value for the batch size $n$ and minimize the number of negative classes $k$ under constraints for $k$.
\end{consequence}

We validate the utility of the introduced $n* =  k* = \sqrt{B}$ rule in further experiments.
Further improvements of the rule are possible, if we take into the account the dependence of the optimal values on $B$.

\section{Experiments}

In this section, we investigate the interaction between batch size and the number of negative samples in training objectives based on negative sampling. The primary goal of the experiments is to analyze how different configurations affect optimization dynamics and convergence behavior.
We consider both synthetic and real data experiments to explore how our theoretical results should be applied in practical recommender systems.

\subsection{Evaluation metrics}

In experiments, we compare the final quality that manifests as NDCG@10 in the end of the training or at specific training iteration and the convergence speed that we measure with \emph{Area Under the Validation Loss Curve (AUL)}, defined in the next paragraph.

The AUL is formally defined as $
\text{AUL} = \sum_{t=1}^{T} \mathcal{L}_{\text{val}}(t) \cdot \Delta t
$,
where $\mathcal{L}_{\text{val}}(t)$ represents the validation loss at training step $t$, and $T$ denotes the total number of processed batches, $\Delta t$ is the time difference between iterations.
A lower AUL implies that the loss curve decays faster and maintains a lower baseline throughout training, indicating superior optimization efficiency.

\subsection{Synthetic Experiments}

We first conduct synthetic experiments. This setup allows us to analyze the optimization dynamics without the confounding effects of specific architectures or data.

\paragraph{Datasets and Model.} We generate a synthetic dataset consisting of $N = 100 000$ samples distributed across $K = 50 000$ distinct classes. 
Each sample $\vecH$ is represented by a continuous feature vector from a multivariate Gaussian. 
The target is generated under a multinomial logistic regression model.
Our model uses the same linear layer followed by softmax architecture. 
During optimization under fixed memory budget we uniformly at random perform objects and negatives sampling with the cross-entropy objective.

\paragraph{Hyperparameters and Configurations.} We ran a set of four different memory budget $B \in \{2^{4}, 2^{6}, 2^{10}, 2^{14}\}$. For a fixed value, we varied the $(n, k)$ pairs, from high-batch/low-sample regimes to low-batch/high-sample regimes.

\paragraph{Optimization Dynamics.} We visualize the validation loss trajectories over the course of training for the selected configurations in Figure~\ref{fig:synthetic_loss}. 

\begin{figure*}[h]
    \centering
    \includegraphics[width=1.7\columnwidth]{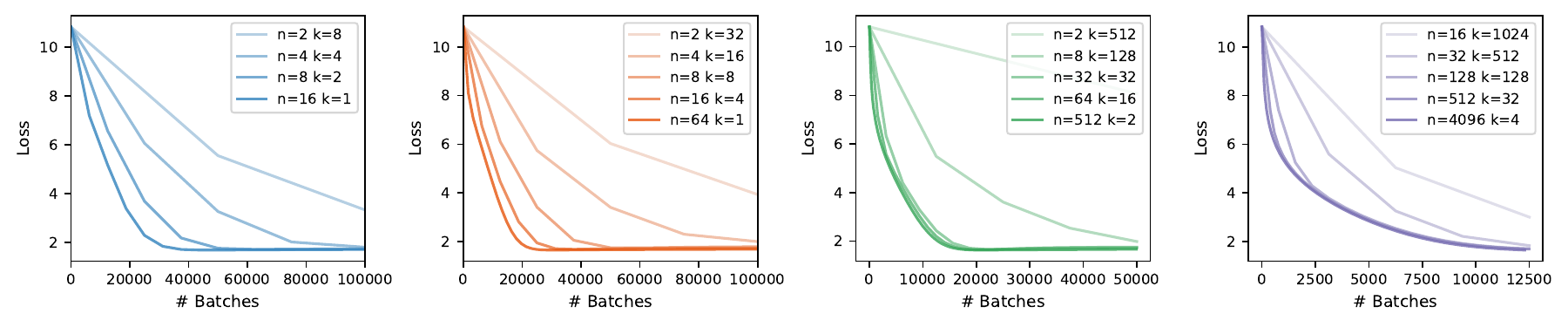}
    \caption{Validation loss dynamics over training iterations for various $(n, k)$ configurations. Each subplot from left to right represents a fixed memory budget $B \in \{2^{4}, 2^{6}, 2^{10}, 2^{14}\}$.}
    \label{fig:synthetic_loss}
\end{figure*}

The convergence  significantly accelerates as we increase the batch size $n$ at the expense of reducing the number of negative samples $k$. 
The configurations dominated by larger batch sizes achieve lower loss values much earlier in the optimization process, confirming our theoretical results.

\paragraph{Convergence Speed Quantification.}

Figure~\ref{fig:synthetic_nk} illustrates the relationship between the computed AUL and the structural ratio $\frac{n}{k}$ across all evaluated configurations.

\begin{figure}[h]
    \centering
    \includegraphics[width=0.8\columnwidth]{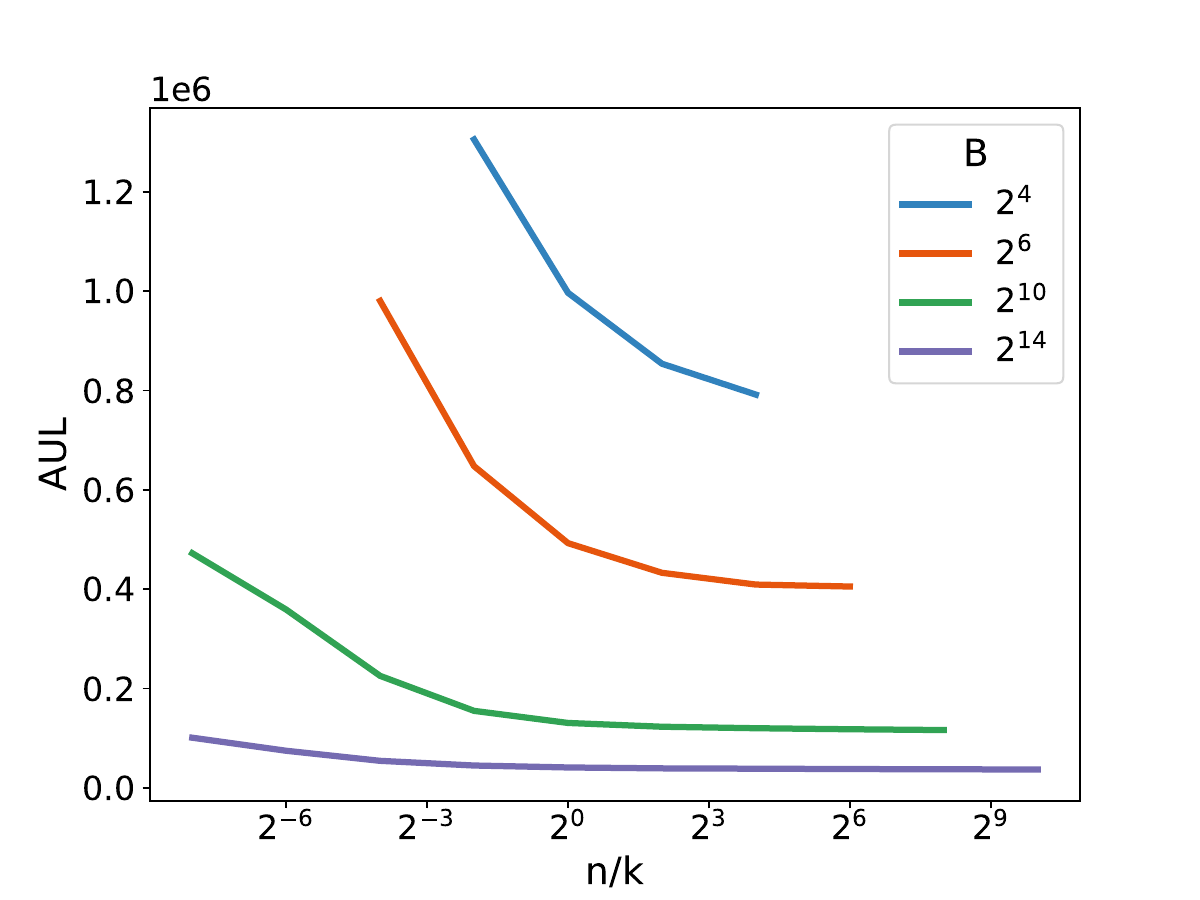}
    \caption{Area Under the Validation Loss Curve (AUL) as a function of the $\frac{n}{k}$ ratio for different memory budget $B$ configurations.}
    \label{fig:synthetic_nk}
\end{figure}

The empirical evidence in Figure~\ref{fig:synthetic_nk} shows that the AUL is a strictly decreasing function of the $\frac{n}{k}$ ratio. This behavior confirms that given a rigid hardware memory budget $B$, it is consistently more advantageous to maximize the batch size $n$ while scaling down the number of negative classes $k$. This structural configuration yields the most rapid objective minimization and optimal convergence behavior.

\subsection{Main Experiments}

The main experiments were conducted on sequential recommendation benchmarks using the Recommender Systems model family based on SASRec~\cite{kang2018self}. 
We evaluate the proposed training configurations on several widely used datasets, including MovieLens-1M, Gowalla, Netflix and MovieLens-20M.

We evaluate three different configurations of the pair batch size and number of classes $(n, k)$, specifically $(32, 512)$, $(64, 256)$, $(128, 128)$.


In experiments, we compare the standard cross-entropy objective (Vanilla) and a cross-entropy objective with additional logit correction terms introduced in Theorem ~\ref{th:unbias} (Unbiased).

We repeat experiments five times with different random seeds, and all reported metrics are averaged across these runs.
Considered optimizers are SGD and Adam.

For the datasets, the learning rate sequences from Table~\ref{tab:lr_values} are used.

\begin{table}[h]
\centering
\begin{tabular}{lll}
\toprule
Dataset & Optimizer & Learning Rates \\
\midrule
MovieLens-1M & SGD  & $1e\!-\!1,\; 3e\!-\!1,\; 1.0,\; 2.0$ \\
MovieLens-1M & Adam & $1e\!-\!3,\; 3e\!-\!3,\; 1e\!-\!2,\; 2e\!-\!2$ \\
\midrule
Gowalla      & SGD  & $1e\!-\!2,\; 3e\!-\!2,\; 1e\!-\!1$ \\
Gowalla      & Adam & $1e\!-\!4,\; 3e\!-\!4,\; 1e\!-\!3$ \\
\midrule
Netflix      & SGD  & $1e\!-\!2,\; 3e\!-\!2,\; 1e\!-\!1$ \\
Netflix      & Adam & $1e\!-\!4,\; 3e\!-\!4,\; 1e\!-\!3$ \\
\midrule
MovieLens-20M      & SGD  & $1e\!-\!2$ \\
MovieLens-20M      & Adam & $1e\!-\!4$ \\
\bottomrule
\end{tabular}
\caption{Learning rate grids evaluated for each dataset and optimizer. For MovieLens-20M, a single robust learning rate was utilized to balance the computational overhead of the multi-dataset evaluation suite.}
\label{tab:lr_values}
\end{table}

\subsection{Evaluation}

To analyze optimization behavior, we consider both final recommendation quality and training dynamics throughout optimization, as well as plots for optimization trajectories for different learning rates and $(n, k)$ configurations.

\subsection{Computational Efficiency and Hyperparameter Sensitivity}

We begin analysis by examining the practical trade-offs associated with varying the number of negative samples $k$. Figure~\ref{fig:training_time_vs_k} illustrates the average training time per epoch as a function of $k$. 
As expected, increasing the number of classes/negative samples $k$ results in a superlinear growth in computational overhead, permitting correct comparison for this pairs.
Thus, in further experiments we consider $k \leq n$ only.

\begin{figure}[t]
    \centering
    \includegraphics[width=0.8\columnwidth]{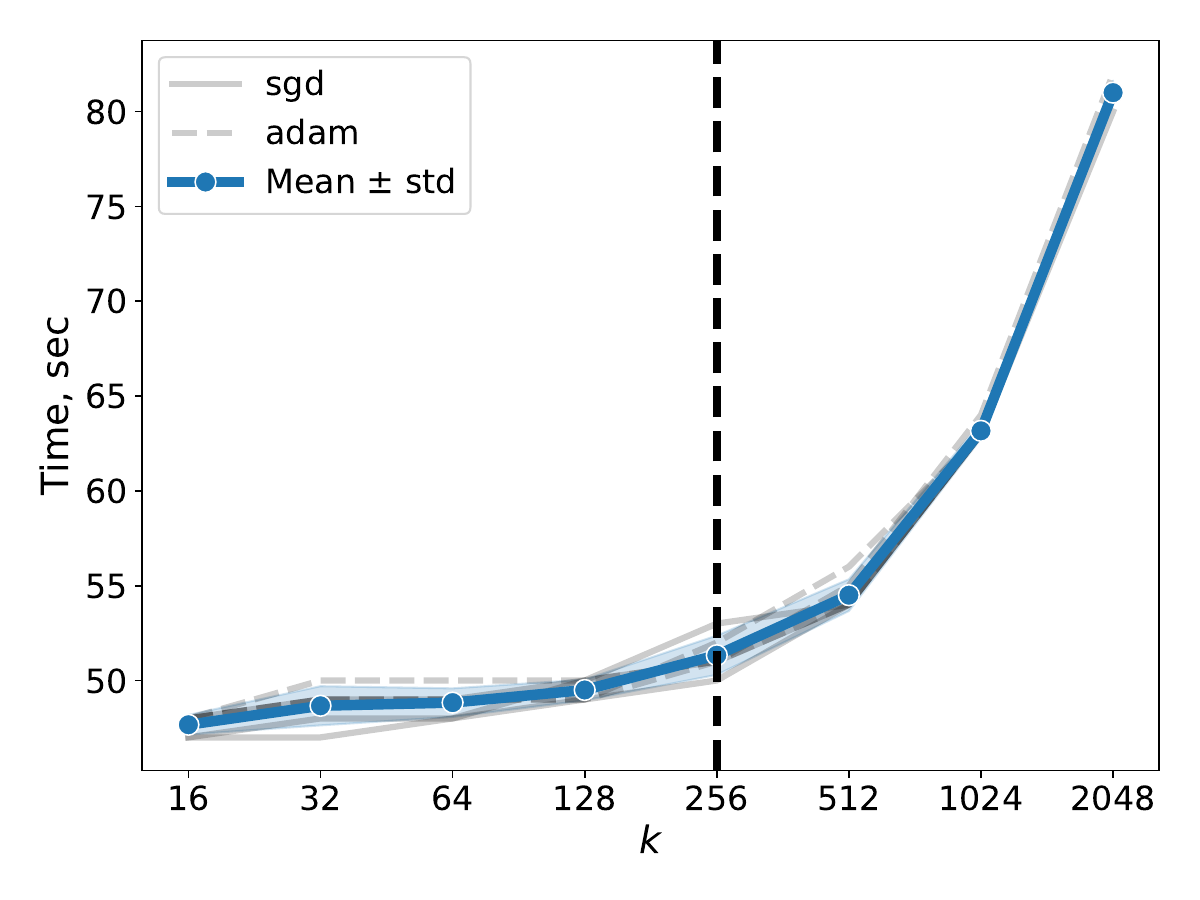}
    \caption{Average training time per epoch as a function of the number of negative classes $k$ for a fixed batch size $n = 256$. The vertical black line corresponds to $k = n$. Each curve represents a distinct seed run.}
    \label{fig:training_time_vs_k}
\end{figure}

We also varied the learning rate to select the best one. 
The convergence behavior shifts across $(n, k)$ pairs even within the same dataset, while in most cases the lowest one is the best.
As for optimizers, Adam requires a much tighter and lower learning rate grid than SGD to prevent divergence. 
In further experiments, we report results for optimal learning rates.





\subsection{Main Results and Convergence Analysis}

Table~\ref{tab:primary_result} summarizes the final test NDCG@10 and the corresponding AUL values for all combinations of datasets, optimizers, and $(n, k)$~pairs.

\begin{table}[h]
\centering
\begin{tabular}{lllll}
\toprule
Optimizer & $n$ & $k$ & NDCG@10 & AUL $(\times 10^3)$ \\
\midrule

\multicolumn{5}{c}{\textbf{MovieLens-1M}} \\
\midrule
 & $32$  & $512$ & $0.0339\pm0.0023$ & $\textbf{46.5}\pm\textbf{3.0}$\\
SGD  & $64$  & $256$ & $0.0372\pm0.0018$ & $58.8\pm0.04$ \\
 & $128$ & $128$ & $0.0363\pm0.0036$ & $61.8\pm0.1$ \\
\cmidrule(lr){1-5}
 & $32$  & $512$ & $0.0361\pm0.0023$ & $43.8\pm0.1$ \\
Adam & $64$  & $256$ & $0.0373\pm0.0033$ & $50.6\pm0.5$ \\
 & $128$ & $128$ & $0.0359\pm0.0025$ & $\textbf{41.1}\pm\textbf{0.4}$ \\

\midrule
\rowcolor{verylightgray}
\multicolumn{5}{c}{\textbf{Gowalla}} \\
\midrule
\rowcolor{verylightgray}
 & $32$  & $512$ & $0.0349\pm0.0066$ & $194.8\pm1.7$ \\
\rowcolor{verylightgray}
SGD  & $64$  & $256$ & $0.0498\pm0.0208$ & $116.9\pm12.0$ \\
\rowcolor{verylightgray}
 & $128$ & $128$ & $0.0369\pm0.0047$ & $\textbf{71.9}\pm\textbf{2.8}$ \\
\cmidrule(lr){1-5}
\rowcolor{verylightgray}
 & $32$  & $512$ & $0.0649\pm0.0031$ & $158.0\pm0.1$ \\
\rowcolor{verylightgray}
Adam & $64$  & $256$ & $0.0858\pm0.0041$ & $94.1\pm0.1$ \\
\rowcolor{verylightgray}
 & $128$ & $128$ & $0.0889\pm0.0037$ & $\textbf{43.7}\pm\textbf{0.1}$ \\

\midrule
\multicolumn{5}{c}{\textbf{Netflix}} \\
\midrule
 & $32$  & $512$ & $0.0389\pm0.0025$ & $989.9\pm202.6$ \\
SGD  & $64$  & $256$ & $0.0347\pm0.0004$ & $791.9\pm0.4$ \\
 & $128$ & $128$ & $0.0375\pm0.0006$ & $\textbf{353.7}\pm\textbf{54.1}$ \\
\cmidrule(lr){1-5}
 & $32$  & $512$ & $0.0378\pm0.0008$ & $1191.1\pm194.7$ \\
Adam & $64$  & $256$ & $0.0377\pm0.0012$ & $664.7\pm51.3$ \\
 & $128$ & $128$ & $0.0388\pm0.0019$ & $\textbf{259.2}\pm\textbf{96.7}$ \\

\midrule
\rowcolor{verylightgray}
\multicolumn{5}{c}{\textbf{MovieLens-20M}} \\
\midrule
\rowcolor{verylightgray}
 & $32$  & $512$ & $0.0176\pm0.0014$ & $233.1\pm0.1$ \\
\rowcolor{verylightgray}
SGD  & $64$  & $256$ & $0.0138\pm0.0004$ & $121.0\pm0.04$ \\
\rowcolor{verylightgray}
 & $128$ & $128$ & $0.0093\pm0.0010$ & $\textbf{61.5}\pm\textbf{0.03}$ \\
\cmidrule(lr){1-5}
\rowcolor{verylightgray}
 & $32$  & $512$ & $0.0189\pm0.0010$ & $414.1\pm0.1$ \\
\rowcolor{verylightgray}
Adam & $64$  & $256$ & $0.0197\pm0.0018$ & $218.6\pm0.04$ \\
\rowcolor{verylightgray}
 & $128$ & $128$ & $0.0201\pm0.0000$ & $\textbf{111.6}\pm\textbf{0.02}$ \\

\bottomrule
\end{tabular}
\caption{Test NDCG@10 metric values for different $(n, k)$ configurations across datasets and optimizers. AUL values are reported in thousands. The best train speed convergence value for each configuration is highlighted in bold.}
\label{tab:primary_result}
\end{table}



The empirical results show that in the vast majority of scenarios, especially for a better performing Adam optimizer, a configuration with less $k$ achieves the minimum AUL.

To provide a deeper insight into these optimization pathways, we visualize a representative subset of the validation loss and NDCG@10dynamics in Figures~\ref{fig:loss_dynamics_grid} and \ref{fig:ndcg_dynamics_grid}. The x-axis tracks the exact number of processed batches, while the y-axis shows the comparative loss values.
The trajectories terminate at different lengths because training is performed for a fixed number of epochs, resulting in different numbers of optimization steps due to varying batch sizes.
These trajectories also confirm that smaller $k$-s combined with larger batch sizes $n$ lead to a steeper loss decay. 

\begin{figure}[t]
    \centering

    \includegraphics[width=\columnwidth]{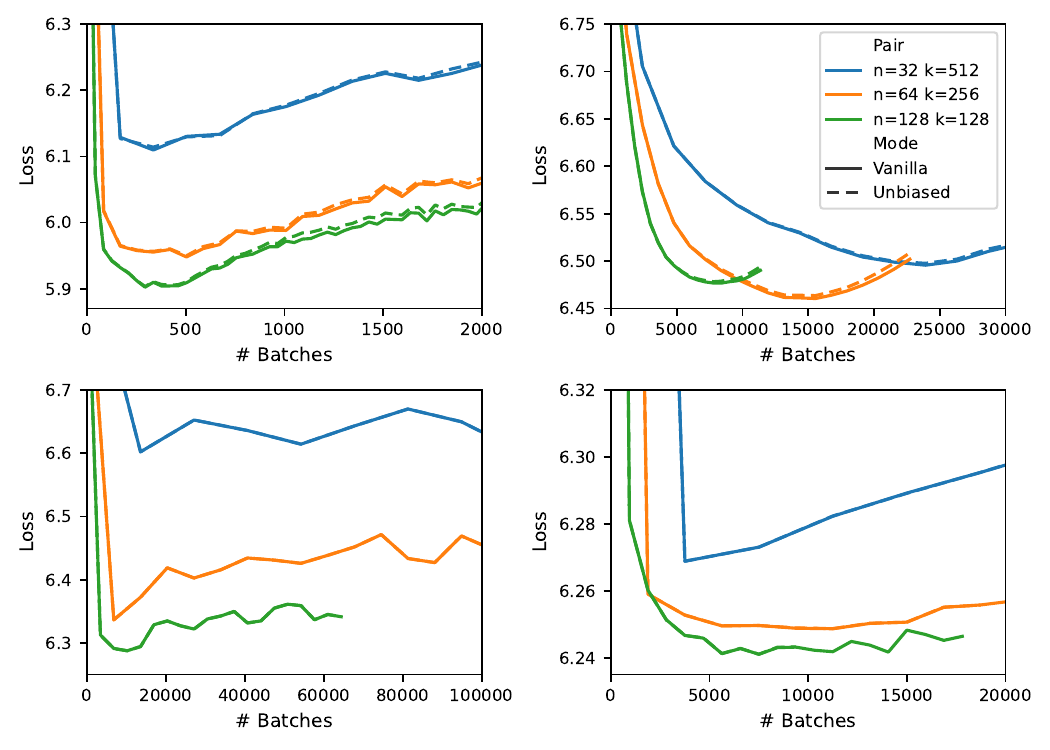}









    \caption{Validation loss dynamics across datasets MovieLens-1M (Top Left), Gowalla (Top Right), Netflix (Bottom Left) and MovieLens-20M (Bottom Right). }
    \label{fig:loss_dynamics_grid}
\end{figure}

\begin{figure}[t]
    \centering

    \includegraphics[width=\columnwidth]{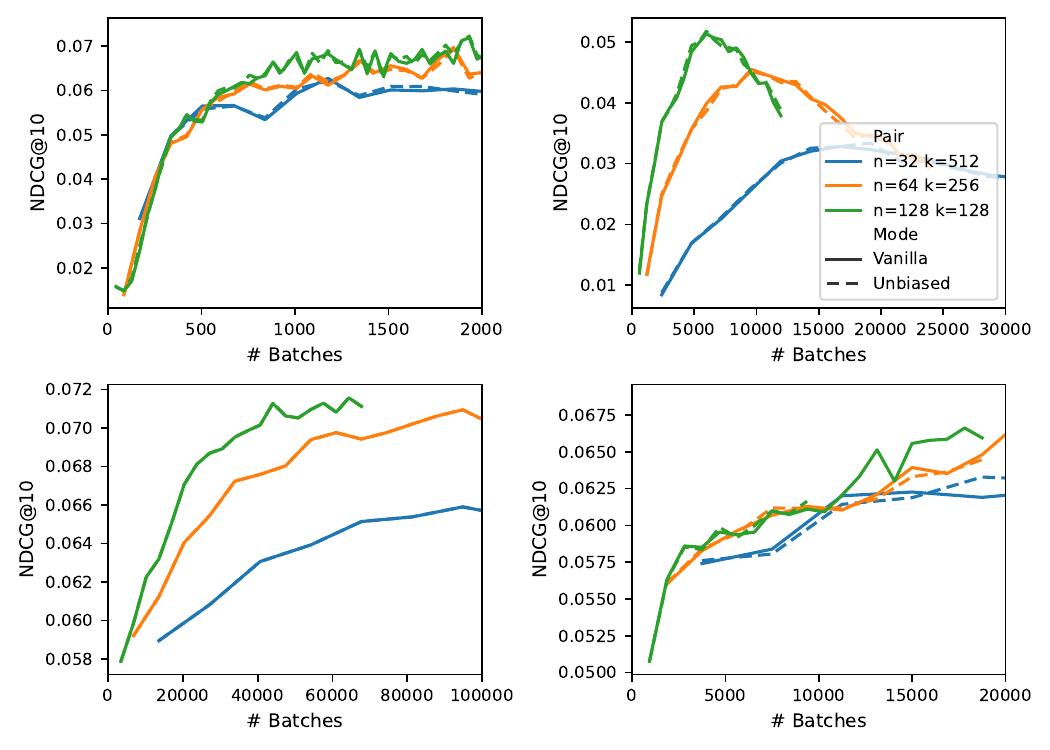}

    \caption{Validation NDCG@10 metric dynamics across datasets MovieLens-1M (Top Left), Gowalla (Top Right), Netflix (Bottom Left) and MovieLens-20M (Bottom Right)}
    \label{fig:ndcg_dynamics_grid}
\end{figure}



\subsection{Discussion and Key Takeaways}

A negative result obtained during evaluation concerns the choice of the objective function. 
When comparing the vanilla cross-entropy loss against the unbiased cross-entropy loss with correction terms (Theorem~\ref{th:unbias}), we observed no statistically significant difference in terms of optimization speed or final recommendation metrics. 
The corrected objective neither accelerates nor slows the training process, and its loss curves completely overlap with the standard baseline. This suggests that while class-imbalance correction is theoretically sound, its practical impact on optimization dynamics within sequential recommendation architectures (like SASRec) remains negligible. 

Contrary, the reduction of the negative sampling class size $k$ in favor of an increased batch size $n$ under a constrained memory budget consistently accelerates convergence. 
Thus, the accelerated convergence is purely driven by the structural optimization benefits of the $(n, k)$ trade-off rather than modifications to the loss landscape.

\section{Conclusion}

We addressed a question in large-scale recommender training on how to optimally allocate a fixed memory budget between batch size and the number of sampled negatives in sampled-softmax objectives to achieve faster convergence. 
The paper formalizes this trade-off as a constrained stochastic optimization problem with $B = n \cdot k$, where $n$ controls mini-batch noise and $k$ controls class-sampling~noise.

Our key contribution is a theoretical analysis showing that these two noise sources play different roles. 
Under standard smoothness and variance assumptions, we demonstrate that we should maximize the number of objects in batch with sampling small number of classes.
This allocation yields near-optimal convergence behavior, providing an actionable guideline.

Extensive experiments on MovieLens-1M and Gowalla with various memory budgets and optimizers confirm our theoretical findings.
There, configurations with minimum possible values of $k$ consistently outperform imbalanced ones in both convergence speed and final recommendation quality across the optimizers considered, including pure SGD and Adam. 

Beyond recommender systems, our results highlight a broader principle for memory-constrained stochastic optimization.
We believe that the presented framework enables direct practical implementation, supports existing empirical claims, and enables further theoretical analysis in more complex training settings.

\section{Acknowledgements} The research was supported by the Russian Science Foundationgrant No. 25-11-00355




\appendix

\section{Proofs}

\subsection{Proof for Theorem~\ref{th:unbias}}
\begin{proof}

The correct estimate of the gradient for cross-entropy loss:
\begin{equation*}
    \mathbf{g} = - \frac{1}{n} \sum\limits_{j \in \batchSet} \left(\mathbf{p_j} - \mathbf{y_j} \right) \  \mathbf{x_j^T}.
\end{equation*}
Instead of $\mathbf{p}$ we use $\mathbf{p}'$, $\mathbb{E}\mathbf{p'} \ne \mathbb{E}\mathbf{p}$, so the gradient estimate is biased. 

The expectations of the sums in the denominators of softmax~\eqref{eq:softmax_via_h} for $\mathbf{p}$ and $\mathbf{p}'$ are respectively:
\begin{equation*}
\mathbb{E} \left[ \sum\limits_{i=1}^{K} h_i \right]= h_{\mathrm{pos}}+(K-1)h_{\mathrm{neg}},\,
\mathbb{E} \left[ \sum\limits_{i \in \mathbb{S}} h_i \right] =  h_{\mathrm{pos}} + (k - 1) h_{\mathrm{neg}} .  
\end{equation*}
If $i \in \mathbb{S}$:
\begin{equation*}
\frac{\mathbb{E}[p_i]}{\mathbb{E}[p_i']} = \frac{h_{\mathrm{pos}}+(k-1)h_{\mathrm{neg}}}{h_{\mathrm{pos}} + (K-1) h_{\mathrm{neg}}} = \frac{\alpha + k - 1}{\alpha + K - 1}.
\end{equation*}

\noindent This is the multiplicative correction factor for the target class. For negative classes, taking into account that they are zeroed with probability $\frac{k}{K}$:

\begin{equation*}
\frac{\mathbb{E}[p_i]}{\mathbb{E}[p_i']} = \frac{K}{k} \cdot \frac{\alpha + k - 1}{\alpha + K - 1}.
\end{equation*}
Since $\alpha + K > \alpha+k \gg 1$, we neglect $-1$ in all expressions and obtain the statement of the theorem.
\end{proof}

\subsection{Proof for Theorem~\ref{th:full_gradient_variance}}

\begin{proof}

The gradient admits the following decomposition with two high level terms:
\begin{align}
\label{eq:gradient_estimation_34}
\mathbf{g} 
&= - \frac{1}{n} \sum_{j \in \batchSet} ( \  \mathbf{p''_j} - \mathbf{y_j)} \  \mathbf{x_j^T} = - \frac{1}{n} \sum_{j \in \batchSet} (\mathrm{diag}(\mathbf{c_j}) \  \mathbf{p'_j} - \mathbf{y_j)} \  \mathbf{x_j^T} \notag \\
&= - \frac{1}{n} \sum_{j \in \batchSet} \mathrm{diag}(\mathbf{c_j}) \  \mathbf{p'_j} \ \mathbf{x_j^T} + \frac{1}{n} \sum\limits_{j \in \batchSet} \mathbf{y_j} \  \mathbf{x_j^T}.
\end{align}

The proof would consists of three parts: separate estimation of the first and the second terms via two lemmas, and union of these estimates into the final results.


\begin{lemma}
In the notation of Theorem above, with $\pCoef = 3 (1 - e^{-\sigma_l^2})$:
\begin{equation}
\label{eq:var_p_for compare}
    \sum_{j \in \batchSet} \mathrm{diag}(\mathbf{c_j}) \  \mathbf{p'_j} \ \mathbf{x_j^T} \approx \mathrm{const} + \frac{\pCoef}{B^2 K} n.
\end{equation}
\end{lemma}

\begin{proof}

With $\vecX_j$ deterministic and $\vecP''_j$ a random vector, our goal is to obtain the variance of $\vecP''_j$.

For a single element in the batch $(\vecX, \vecY)$, consider the coordinates in $\vecY$ corresponding to negative and target classes separately. 
For a negative class $i_0$, according to the previous theorem:
\begin{align*}
\hat{c}_{\mathrm{neg}} 
&= \frac{K}{k} \cdot \frac{h_{\mathrm{pos}} + k \hat{h}_{\mathrm{neg}}}{h_{\mathrm{pos}} + K \hat{h}_{\mathrm{neg}}}, p_{i_0}' 
= \frac{h_{i_0}}{h_{\mathrm{pos}} + k \hat{h}_{\mathrm{neg}}} \mathrm{Be}\left(\frac{k}{K}\right), \\
\hat{h}_{\mathrm{neg}}
&= \frac{1}{k} \sum_{i \in \mathbb{S}^{\mathrm{neg}}} h_i, 
\end{align*}
where $\mathrm{Be}\left(\frac{k}{K}\right)$ is a Bernoulli random variable characterizing the probability that the class is included in the set $\mathbb{S}$. Here and below, we assume $K \gg k \gg 1$, and, thus, $k - 1 \approx k$.

Then:
\begin{equation*}
\hat{c}_{\mathrm{neg}} p_{i_0}' = \frac{K}{k} \cdot \frac{h_{i_0}}{h_{\mathrm{pos}} + K \hat{h}_{\mathrm{neg}}} \cdot \mathrm{Be}\left(\frac{k}{K}\right)
\end{equation*}

Now, let us derive the variance for it:
\begin{align*}
\mathbb{D} ( \hat{c}_{\mathrm{neg}} p_{i_0}' ) 
&= \mathbb{D} \left[ \frac{K}{k}  \frac{h_{i_0}}{h_{\mathrm{pos}} + K \hat{h}_{\mathrm{neg}}}  \mathrm{Be}\left(\frac{k}{K}\right) \right] \\
&= h_{i_0}^2  \mathbb{D} \left[ \frac{1}{\frac{k}{K} h_{\mathrm{pos}} + \sum_{i \in \mathbb{S}^{\mathrm{neg}}} h_i}  \mathrm{Be}\left(\frac{k}{K}\right) \right].
\end{align*}

Next, we use the formula for the variance of independent random variables:
\begin{align}
\label{eq:var_multip_indep}
    \mathbb{D}[XY] &= \mathbb{D}X \mathbb{D}Y + (\mathbb{E}X)^2 \mathbb{D}Y + (\mathbb{E}Y)^2 \mathbb{D}X \notag \\
    &= (\mathbb{E}X^2)\mathbb{D}Y + (\mathbb{E}Y)^2 \mathbb{D}X.
\end{align}

For the Bernoulli variable $X = \mathrm{Be}\left(\frac{k}{K}\right)$, $\mathbb{E} X^2 = \frac{k}{K}$, $\mathbb{D} X = \frac{k}{K}\left(1-\frac{k}{K}\right)$; we investigate now $Y$. Introduce the variables $a = \frac{k}{K} h_{\mathrm{pos}}$, $S = \sum_{i \in \mathbb{S}^{\mathrm{neg}}} h_i$.
Then, using~A3:
\begin{equation*}
    Y = \frac{1}{h_{\frac{k}{K} \mathrm{pos}} + \sum_{i \in S} h_i} = \frac{1}{a+S} \approx \frac{1}{S}.
\end{equation*}

Since the logits $l_i$ are normally distributed, their exponentials $h_i$ follow a log‑normal distribution. The distribution of a sum of independent log‑normal variables is not analytical.
CLT has slow convergence here due to asymmetry, requiring a large $k$. 
However, under assumptions A1 and A2, the \textit{Fenton–Wilkinson} approximation~\cite{marlow1967normal} works well, i.e., approximating the sum of log‑normals by a single log‑normal with the same expectation and variance as the sum, provide better results than CLT as evident from Figure~\ref{fig:lognorm_sum_approximations}.

\begin{figure}[h]
  \centering
  \includegraphics[width=0.8\linewidth]{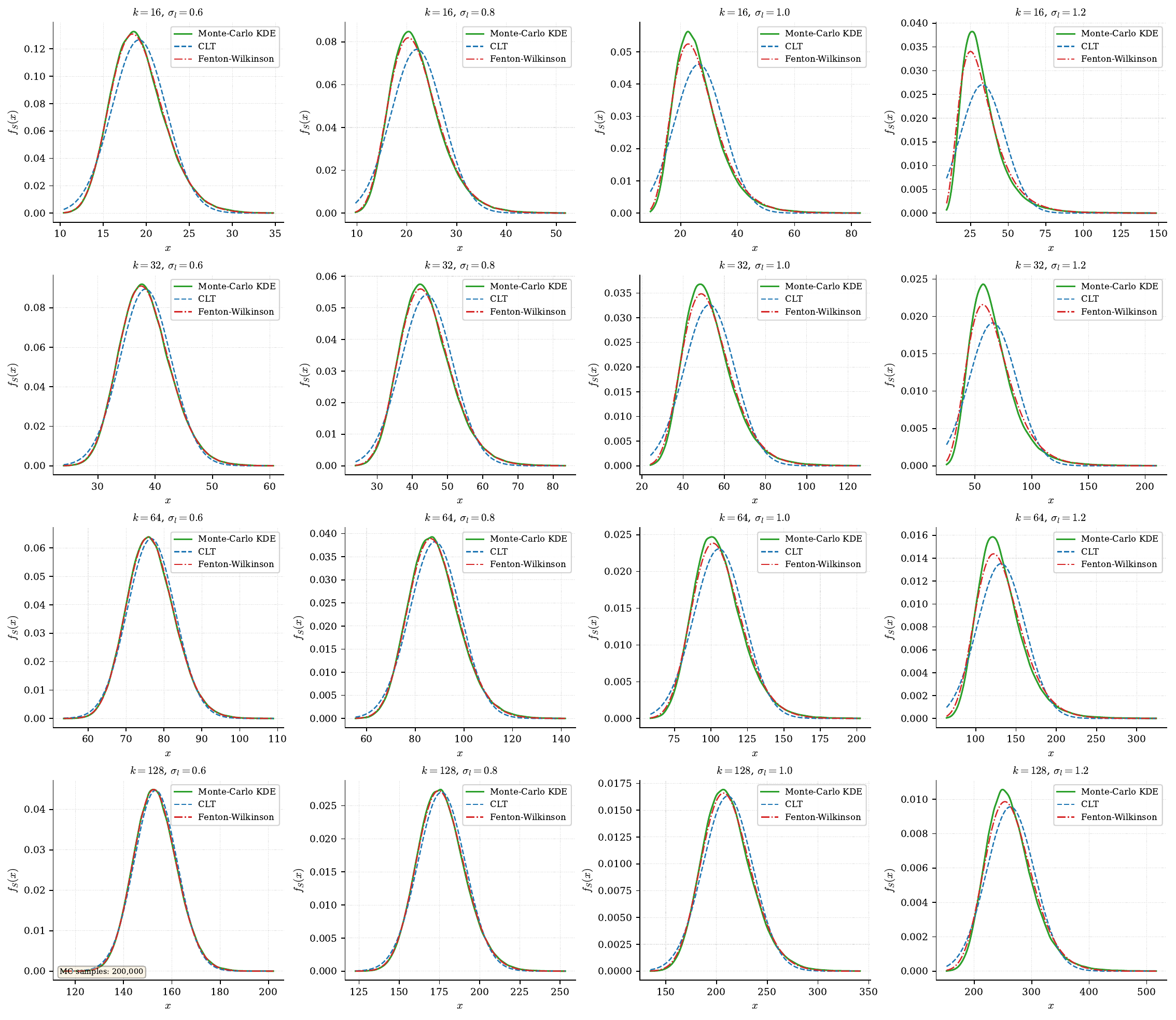}
  \caption{Approximation of the sum of $k$ lognormal random variables $\mathrm{Lognorm}(0, \sigma_l)$ via CLT and Fenton--Wilkinson. The true distribution is reconstructed using Monte Carlo KDE. Better to view in zoom.}
  \label{fig:lognorm_sum_approximations}
  \Description{}
\end{figure}

The Fenton–Wilkinson approximation formulas:
\begin{align*}
    &\sum_{i \in \mathbb{S}^{\mathrm{neg}}} h_i \approx \hat{S} \sim \mathrm{Lognorm}(m, v), \\
    &m = \ln \mathbb{E}[S] - \frac{v}{2}
  = \ln \left(k e^{\sigma_l^2 / 2}\right) - \frac{v}{2}
  = \ln k + \frac{\sigma_l^2}{2} - \frac{v}{2}, \\
    &v = \ln \left(1 + \frac{\mathbb{D}[S]}{\mathbb{E}[S]^2}\right)
     = \ln \left(1 + \frac{e^{\sigma_l^2} - 1}{k}\right).
\end{align*}
Using the properties of the lognormal distribution,  for $Y$ we obtain:
\begin{equation*}
    Y \approx \frac{1}{\hat{S}} \sim \mathrm{LogNorm}(-m, v).
\end{equation*}
Then for the squared expectation we have:
\begin{align*}
\mathbb{E}\left[\frac{1}{\hat{S}}\right] 
&= \exp\left(-m + \frac{v}{2}\right) = \exp\left(-\ln k - \frac{\sigma_l^2}{2} + \frac{v}{2} + \frac{v}{2}\right), \\
\left(\mathbb{E}\left[\frac{1}{\hat{S}}\right]\right)^2 
&= \frac{(k + e^{\sigma_l^2} - 1)^2}{k^4} e^{-\sigma_l^2}.
\end{align*}
For the variance:
\begin{align*}
\mathbb{D}\left[\frac{1}{\hat{S}}\right]  
= e^{-2m+v} \left(e^{v} - 1\right) 
= \frac{(k + e^{\sigma_l^2} - 1)^2}{k^5} \left(1 - e^{-\sigma_l^2} \right).
\end{align*}

Substituting into~\eqref{eq:var_multip_indep}:
\begin{align*}
( \mathbb{D}X + (\mathbb{E}X)^2) \mathbb{D}Y 
&= \frac{(k + e^{\sigma_l^2} - 1)^2}{k^4 K} e^{-\sigma_l^2} (e^{\sigma_l^2} - 1), \\
(\mathbb{E}Y)^2 \mathbb{D}X 
&= \frac{(k + e^{\sigma_l^2} - 1)^2}{k^3 K} e^{-\sigma_l^2} \left(1 - \frac{k}{K}\right), \\
\mathbb{D}[XY] &= \frac{e^{-\sigma_l^2}}{K} \cdot \frac{(k + e^{\sigma_l^2} - 1)^2}{k^3 } \left( \frac{e^{\sigma_l^2} - 1}{k} + 1 - \frac{k}{K} \right).
\end{align*}

Repeating the derivation for the target class $i_0$ analogously to the negative case:
\begin{align*}
\hat{c}_{\mathrm{pos}} 
&= \frac{h_{\mathrm{pos}} + k \hat{h}_{\mathrm{neg}}}{h_{\mathrm{pos}} + K \hat{h}_{\mathrm{neg}}}, \,
p_{i_0}' 
= \frac{h_{i_0}}{h_{\mathrm{pos}} + k \hat{h}_{\mathrm{neg}}}, \\
\hat{c}_{\mathrm{pos}} p_{i_0}'
&= \frac{h_{i_0}}{h_{\mathrm{pos}} + K \hat{h}_{\mathrm{neg}}}, \\
\mathbb{D} ( \hat{c}_{\mathrm{pos}} p_{i_0}' ) 
&= \mathbb{D} \left[\frac{h_{i_0}}{h_{\mathrm{pos}} + K \hat{h}_{\mathrm{neg}}} \right] \\
&= \frac{k^2}{K^2} h_{i_0}^2 \mathbb{D} \left[ \frac{1}{\frac{k}{K} h_{\mathrm{pos}} + \sum_{i \in \mathbb{S}^{\mathrm{neg}}} h_i} \right].
\end{align*}

The first factor is infinitesimally small compared to the negative coordinates because $k \ll K$; the contribution to the final gradient norm comes precisely from the negative coordinates, and we focus on them.

If there is not one but $n$ elements in the batch, we have:
\begin{align*}
    \mathbb{D}\left[ \frac{1}{n} \sum\limits_{j \in \batchSet} p_{j, i_0}'' \right] 
    &= \frac{1}{n}\mathbb{D}[XY] = \frac{k}{B}\mathbb{D}[XY] \\ 
    &= \frac{e^{-\sigma_l^2}}{BK} \cdot \frac{(k + e^{\sigma_l^2} - 1)^2}{k^2} \left( \frac{e^{\sigma_l^2} - 1}{k} + 1 - \frac{k}{K} \right).
\end{align*}

For $k \gg 1$:
\begin{equation*}
    \frac{(k + e^{\sigma_l^2} - 1)^2}{k^2} = \left( 1 + \frac{e^{\sigma_l^2} - 1}{k}\right)^2 = 1 + \frac{2 (e^{\sigma_l^2} - 1) }{k} + O \left( \frac{1}{k^2} \right), 
\end{equation*}
Thus, for $K \gg k \gg 1$:
\begin{align*}
\label{eq:var_p_for compare_2}
    \mathbb{D}\left[ \frac{1}{n} \sum\limits_{j \in \batchSet} p_{j, i_0}'' \right] 
    &= \frac{e^{-\sigma_l^2}}{BK} \left(1 + \frac{2 (e^{\sigma_l^2} - 1) }{k} + O \left( \frac{1}{k^2} \right)\right) \left(1 + \frac{e^{\sigma_l^2} - 1}{k} - \frac{k}{K} \right) \\ 
    &\approx \mathrm{const} + \frac{3 e^{-\sigma_l^2}}{BK} \frac{e^{\sigma_l^2} - 1}{k} 
    = \mathrm{const} + \frac{3(1 - e^{-\sigma_l^2})}{B^2 K} n. \\
\end{align*}

Thus, for $\pCoef = 3(1 - e^{-\sigma_l^2}) > 0$, we get the Lemma statement.
\end{proof}


\begin{lemma}
In the notation of Theorem above, with $\yCoef = \mathbb{E} x^2_r$:
\begin{equation}
\label{eq:var_y_for_compare}
     \mathbb{D}\left[\frac{1}{n}
        \sum_{j \in \batchSet} y_{j, i_0} x_{j, r} \right] \approx \frac{\yCoef}{K n}.
\end{equation}

\end{lemma}

\begin{proof}

Since the examples in the mini-batch are independent, we have
\[
    \mathbb{D}\left[\frac{1}{n}
        \sum_{j \in \batchSet} y_{j, i_0} x_{j, r}
    \right] =
    \frac{1}{n}
    \mathbb{D}\left[ y_{i_0} x_r \right].
\]
Given the number of classes and inputs, we can assume that $y_{i_0}$ and $x_r$ are independent. 
The variance formula for a product of independent random variables give us
\[
    \mathbb{D}\left[y_{i_0} x_r \right] \approx
    \mathbb{E}[y_{i_0}^2] \mathbb{D}[x_r]
    + \left(\mathbb{E} x_r\right)^2
    \mathbb{D}[y_{i_0}].
\]
Since $y_{i_0}$ is Bernoulli, $y_{i_0}^2 = y_{i_0}$. Therefore,
\[
    \mathbb{E}[y_{i_0}^2]
    =
    \mathbb{E}[y_{i_0}]
    =
    \pi_{i_0}, \,
    \mathbb{D}[y_{i_0}]
    =
    \pi_{i_0}(1 - \pi_{i_0}).
\]
Thus,
\[
    \mathbb{D}\left[
        y_{i_0}x_r
    \right]
    =
    \pi_{i_0}\mathbb{D}[x_r]
    +
    \left(\mathbb{E}x_r\right)^2
    \pi_{i_0}(1-\pi_{i_0}).
\]
Consequently,
\[
    \mathbb{D}\left[
        \frac{1}{n}
        \sum_{j\in\batchSet}
        y_{j,i_0}x_{j,r}
    \right]
    =
    \frac{1}{n}
    \left[
        \pi_{i_0}\mathbb{D}[x_r]
        +
        \left(\mathbb{E}x_r\right)^2
        \pi_{i_0}(1-\pi_{i_0})
    \right].
\]
Under the approximately uniform-class assumption $\pi_{i_0}\approx 1/K$,
this becomes
\[
    \mathbb{D}\left[
        \frac{1}{n}
        \sum_{j\in\batchSet}
        y_{j,i_0}x_{j,r}
    \right]
    \approx
    \frac{1}{n}
    \left[
        \frac{1}{K}\mathbb{D}[x_r]
        +
        \left(\mathbb{E}x_r\right)^2
        \frac{1}{K}
        \left(1-\frac{1}{K}\right)
    \right].
\]
Given $1 - \frac{1}{K} \approx 1$ for large $K$ we obtain:
$
    \mathbb{D}\left[
        \frac{1}{n}
        \sum_{j\in\batchSet}
        y_{j,i_0}x_{j,r}
    \right]
    \approx \frac{1}{n K}
    \mathbb{E}x^2_r
$.
Thus, for $\yCoef = \mathbb{E}x^2_r$, we obtain the statement of the Lemma.
\end{proof}

Combining \eqref{eq:var_p_for compare} and \eqref{eq:var_y_for_compare}, obtained in two previous Lemmas, we get the statement of the theorem with specific positive $\yCoef$ and $\pCoef$:
\[
    \operatorname{Var}(\mathbf{g}_{n,k})
    \leq
    \frac{\yCoef}{K n}
    +
    \frac{\pCoef}{B K k}
\]

\end{proof}

\subsection{Proof for Theorem~\ref{th:optimal_n_k}}

\begin{proof}

The proof consists of solving the constrained optimization problem for positive $\yCoef$ and $\pCoef$:
    \[
        \min_{n > 0, k > 0}
        \left\{
    \frac{\yCoef}{K n} +
    \frac{\pCoef}{B K k}
        \right\}
        \quad
        \text{subject to}
        \quad n k = B.
    \]
Using the constraint $n k = B$, we get $k = \frac{B}{n}$.
Thus, we have a one-dimensional objective
\[
    \phi(n) =
    \frac{A_1}{n}
    + \frac{A_2}{B} n
\]
with positive $A_1 = \frac{\yCoef}{K}$ and $A_2 = \frac{\pCoef}{B K}$.

Differentiating with respect to $n$ and making the derivative zero, we obtain the the candidate minimizer with $\phi'(n) = 0$ and 
$n > 0$:
\[
    n^\star = \sqrt{\frac{A_1}{A_2} B}.
\]
Taking the second derivative, we easily see that it is bigger than $0$.
So, we found a unique global minimizer.
Therefore,
\[
    n^\star =
    \sqrt{\frac{\yCoef}{\pCoef}} B,
    \qquad k^\star =
    \sqrt{\frac{\pCoef}{\yCoef}}.
\]
For typically large $B$ under assumptions about $n, k, B$ we have 
\[
    n^\star \asymp B, \qquad k^\star \asymp 1.
\] 
\end{proof}

\bibliographystyle{ACM-Reference-Format}
\bibliography{sample-base}

\end{document}